\documentclass[final]{ecai} 

\usepackage{latexsym}
\usepackage{amssymb}
\usepackage{amsmath}
\usepackage{amsthm}
\usepackage{booktabs}
\usepackage{enumitem}
\usepackage{graphicx}
\usepackage{color}

\newtheorem{theorem}{Theorem}

\newtheorem{proposition}[theorem]{Proposition}

\usepackage{times}
\usepackage{soul}
\usepackage{url}
\usepackage[hidelinks]{hyperref}
\usepackage[utf8]{inputenc}
\usepackage[small]{caption}
\usepackage{graphicx}
\usepackage{amsmath}
\usepackage{amsthm}
\usepackage{booktabs}
\usepackage{algorithm}
\usepackage{algorithmic}
\usepackage[switch]{lineno}

\usepackage[utf8]{inputenc} % allow utf-8 input
\usepackage[T1]{fontenc}    % use 8-bit T1 fonts
\usepackage{url}            % simple URL typesetting
\usepackage{booktabs}       % professional-quality tables
\usepackage{amsfonts}       % blackboard math symbols
\usepackage{nicefrac}       % compact symbols for 1/2, etc.
\usepackage{microtype}      % microtypography
\usepackage{xcolor}         % colors

\usepackage{natbib}
\usepackage{graphicx}
\usepackage{subfigure}
\usepackage{epstopdf}
\usepackage{amssymb}
\usepackage{mathtools}
\usepackage{amsthm}

\theoremstyle{plain}
\theoremstyle{definition}

\newtheorem{observation}{Observation}
\usepackage{amsmath,amsfonts}

\def\eqref#1{equation~\ref{#1}}
\def\1{\bm{1}}

\def\eps{{\epsilon}}

\DeclareMathAlphabet{\mathsfit}{\encodingdefault}{\sfdefault}{m}{sl}
\SetMathAlphabet{\mathsfit}{bold}{\encodingdefault}{\sfdefault}{bx}{n}

\newcommand{\bfA}{{\bf A}}

\newcommand{\bfI}{{\bf I}}

\newcommand{\bfL}{{\bf L}}

\newcommand{\bfb}{{\bf b}}

\newcommand{\bfe}{{\bf e}}

\newcommand{\bfh}{{\bf h}}

\newcommand{\bfx}{{\bf x}}

\newcommand{\bfr}{{\bf r}}

\newcommand{\bfz}{{\bf z}}

\newcommand{\bftheta}{{\boldsymbol \theta}}

\newcommand{\bfdelta}{{\boldsymbol \delta}}

\DeclareMathOperator*{\argmin}{arg\,min}

\usepackage[capitalize,noabbrev]{cleveref}
\usepackage{multirow, multicol}
\newcommand{\bfSigma}{{\bf \Sigma}}

\usepackage{pifont}% http://ctan.org/pkg/pifont
\newcommand{\cmark}{\ding{51}}%
\newcommand{\xmark}{\ding{55}}%

\newtheorem{example}{Example}

\newcommand{\BibTeX}{B\kern-.05em{\sc i\kern-.025em b}\kern-.08em\TeX}

\begin{document}

%%%%%%%%%%%%%%%%%%%%%%%%%%%%%%%%%%%%%%%%%%%%%%%%%%%%%%%%%%%%%%%%%%%%%%%%

\begin{frontmatter}

%%% Use this command to specify your submission number.
%%% In doubleblind mode, it will be printed on the first page.

\paperid{1618} 

%%% Use this command to specify the title of your paper.

\title{Quadratic Binary Optimization with Graph Neural Networks}

%%% Use this combinations of commands to specify all authors of your 
%%% paper. Use \fnms{} and \snm{} to indicate everyone's first names 
%%% and surname. This will help the publisher with indexing the 
%%% proceedings. Please use a reasonable approximation in case your 
%%% name does not neatly split into "first names" and "surname".
%%% Specifying your ORCID digital identifier is optional. 
%%% Use the \thanks{} command to indicate one or more corresponding 
%%% authors and their email address(es). If so desired, you can specify
%%% author contributions using the \footnote{} command.

\author[A]{\fnms{Moshe}~\snm{Eliasof}}
\author[B]{\fnms{Eldad}~\snm{Haber}}

\address[A]{University of Cambridge}
\address[B]{University of British Columbia}

%%% Use this environment to include an abstract of your paper.

\begin{abstract}
We investigate a link between Graph Neural Networks (GNNs) and Quadratic Unconstrained Binary Optimization (QUBO) problems, laying the groundwork for GNNs to approximate solutions for these computationally challenging tasks. By analyzing the sensitivity of QUBO formulations, we frame the solution of QUBO problems as a heterophilic node classification task. We then propose QUBO-GNN, an architecture that integrates graph representation learning techniques with QUBO-aware features to approximate solutions efficiently. Additionally, we introduce a self-supervised data generation mechanism to enable efficient and scalable training data acquisition even for large-scale QUBO instances. Experimental evaluations of QUBO-GNN across diverse QUBO problem sizes demonstrate its superior performance compared to exhaustive search and heuristic methods. Finally, we discuss open challenges in the emerging intersection between QUBO optimization and GNN-based learning.
\end{abstract}

\end{frontmatter}

%%%%%%%%%%%%%%%%%%%%%%%%%%%%%%%%%%%%%%%%%%%%%%%%%%%%%%%%%%%%%%%%%%%%%%%%

\section{Introduction}
\label{sec:intro}
Solving \emph{Binary Programming} (BP) Problems is of paramount importance in optimization and decision-making. These problems, characterized by binary decision variables, arise in numerous real-world scenarios, such as 
portfolio optimization \cite{black1992global},
manufacturing and supply chain optimization \cite{geunes2005supply}, 
telecommunications network optimization \cite{resende2008handbook},
resource allocation \cite{brown1984concept} and more. 
As the name suggests, Binary Programming (BP) is a type of mathematical problem where decision variables are restricted to take binary values (0 or 1). The discrete nature of BP makes these problems computationally challenging. Specifically, BP problems are known to be NP-complete \cite{karp2010reducibility}. However, the solution of BP problems is crucial for optimizing the allocation of resources and improving operational efficiency. Contemporary solutions involve a range of optimization techniques, ranging from relaxation techniques and linear programming solvers \cite{nowak2005relaxation} to heuristic algorithms and heuristic-based approaches \cite{merz2002greedy}. Emerging technologies, such as quantum computing, show promise in addressing BP problems more efficiently \cite{glover2018tutorial}, offering novel avenues for optimization in the future. However, they are not readily available yet, and their current setups are limited to solving problems with a small number of decision variables. Other recent works proposed to combine reinforcement learning techniques \cite{qiu2022dimes} to learn a differentiable solver, 

Concurrently, in recent years, Graph Neural Networks (GNNs) have emerged as powerful machine learning models, and have been applied to various domains and applications, from social network analysis \cite{kipf2016semi, velickovic2018graph}, bioinformatics \cite{jumper2021highly}, to computer vision and graphics \cite{monti2017geometric}, and more.

While early works suggested leveraging GNNs as an optimization framework to approximate the solution of relaxed versions of the BP problem \citep{schuetz2022combinatorial}, and Constrained Satisfaction Problems \cite{lu2023graph,tonshoff2022one}, these approaches were restricted to a single instance of a problem and required the training from scratch of a network for a specific instance of a problem, or were done in an unsupervised manner \cite{karalias2020erdos}. Extending and departing from this approach, in this paper, we show that it is possible to model and characterize the family of solutions to  BP problems as the solution of a graph learning task, namely, as \emph{heterophilic }node classification. Our findings lay the ground for leveraging the power of Graph Neural Networks in approximating the solutions of BP problems in a supervised manner, allowing to generalize for new, unseen instances of BP problems. 
In this work, we focus on Quadratic Binary Optimization (QUBO) problems, which are a sub-family of BP, and we propose \emph{QUBO-GNN}, a GNN architecture that builds on recent advances in heterophilic node classification with QUBO-aware features as part of its architecture.
Another important aspect when using GNNs (or other data-driven techniques) to solve QUBO problems, is the scarcity of labeled data. While it is possible to use exhaustive search models to generate solutions of problems with a low number of decision variables, generating data for more than a dozen of decision variables becomes computationally intractable. We therefore propose a data generation procedure that stems from our theoretical analysis of the behavior of the solution landscape of QUBO problems. 
Our novel approach not only seeks to deepen the connection between QUBO and neural graph representation learning techniques,  but also holds the potential to offer more efficient and effective methodologies for training a wide range of computationally demanding optimization problems in diverse applications.

{\bf{Contributions:}} This paper makes several contributions on multiple fronts. (i) It establishes a novel connection that links between Graph Neural Networks and Binary Programming problems. This key observation serves as the foundation for exploring the utilization of GNNs in addressing and approximating solutions to computationally hard problems. (ii) We conduct a sensitivity analysis of the Binary Programming problems 
with respect to the input variables. This analysis sheds light on patterns and behavior of the solutions of QUBO learning problems, providing an understanding of how to address these problems using GNNs, by treating the solution of QUBO learning problems as heterophilic node classification. 
(iii) We propose an efficient self-supervised mechanism for generating training data, to address the NP-complete problem of generating data using exhaustive search otherwise. Finally,  (iv) we evaluate QUBO-GNN on varying sizes of QUBO problems and compare its performance with exhaustive search, heuristic, and other GNN approaches.
Collectively, these contributions advance the understanding and utilization of GNNs in the realm of Binary Programming, paving the way for novel and efficient approaches to address this problem.

\section{Problem Formulation}
\label{sec:problemformulation}
In this work, we consider BP problems and focus on the solution of the family of Quadratic Unconstrained Binary Optimization (\emph{QUBO}) 
problems:
\begin{equation}
    \label{eq:bp}
   \bfx_{o} = {\rm arg}\min_{\bfx \in \{0,1\}^k} f(\bfx; \bfb, \bfA) =  \bfx^{\top} \bfA \bfx + \bfx^{\top} \bfb,
\end{equation}
where $\bfx \in \mathbb{B}^{k}$ is a binary decision variables vector, $\bfA \in \mathbb{R}^{k \times k}$
is a real-valued matrix, and $\bfb \in \mathbb{R}^k$ is a given real-valued vector, also sometimes called the observed vector. 
The QUBO problem was thoroughly studied in the literature, and we refer to \cite{beasley1998heuristic, kochenberger2014unconstrained, glover2018tutorial} and reference within for theoretical as well as practical algorithms for its solution.

Obtaining exact solutions for QUBO problems is computationally challenging, due to their discrete settings.
As we discuss in Section \ref{sec:related}, many heuristic-based methods have been proposed to solve this problem. However, despite such heuristics, it is still evident that in many cases, it is difficult to obtain a reasonable solution in polynomial time \cite{kochenberger2014unconstrained}, especially in high dimensions.
In this work, we focus on approximating the QUBO solution in a polynomial time for a subset of QUBO problems, assuming that the matrix $\bfA$ is sparse 
 and fixed, but the observed vector $\bfb$
can vary. Such problems arise in virtually any field where QUBO problems are utilized \cite{black1992global, geunes2005supply, resende2008handbook, brown1984concept}.
For these problems, the matrix $\bfA$
represents the interaction between the variables, and
the vector $\bfb$ is data that varies.

An important observation that we make is that QUBO can be interpreted as a graph problem. 
    Let us consider a graph ${G=(V,E)}$ with a set of $k$ nodes, $V$, and a set of $m$ edges, $E$. The matrix  $\bfA \in \mathbb{R}^{k \times k}$ from \cref{eq:bp} can be interpreted as the adjacency matrix of a graph ${G}$. The
 connectivity of ${G}$ is therefore represented by the sparsity pattern of $\bfA$, and the edge weights are represented by the matrix values. 
    In the context of graphs and GNNs, we can treat $\bfb$ from \cref{eq:bp} as the input node features.
With this observation, we are motivated to explore the utilization of Graph Neural Networks (GNNs) to approximate the solution of QUBO problems. We show that through this approach, it is possible to efficiently obtain approximations to the solutions of the QUBO problem.

\section{Related Work}
\label{sec:related}
We now shortly review related work within both QUBO solvers and
GNN architectures.

\subsection{Popular QUBO Solution Techniques}
\label{sec:relatedBP}
Solving QUBO problems has a wide and rich algorithmic treatment that ranges from exhaustive search methods, relaxation methods, differential equations-based approaches, and grid search methods. We refer the reader to \cite{beasley1998heuristic} and more recently to \cite{goto2019combinatorial} for a review of the algorithms.
We now provide an overview of existing methods, focusing on representative methods that approximate a solution to the QUBO problem in \cref{eq:bp}. We chose these techniques because they are publicly available and commonly used in literature \cite{oshiyama2022benchmark}.

\paragraph{Exhaustive Search.} Overall, the QUBO problem has $2^k$ plausible solutions. While not always practical, the simplest approach is to use an exhaustive search. That is, test each of the possible solutions and choose the one that achieves the minimal value of \cref{eq:bp}. This approach is viable when $k$ is small, but it becomes restrictive as $k$ grows. For instance, a problem of size $k=1024$ has $2^{1024}$ plausible solutions, making exhaustive search impractical.

\paragraph{Tabu Search (TS).} 
Is a family of iterative algorithms that start from some initial guess $\bfx_0$,  and defines a list, ${\cal T}$ of maximum length $\ell$,  dubbed the Tabu list, that is initialized as with $\bfx_0$.
The algorithm then defines a list of neighbors, ${\cal N}$ of $\bfx_0$ as perturbations in various policies \cite{beasley1998heuristic, glover1998tabu, glover2018tutorial}. The simplest perturbation policy is defined as a vector $\tilde{\bfx}_{(i)}$ whose entries are identical to $\bfx_0$, with the exception of the $i$-th index.
The neighbors list is then compared to the Tabu list, and 
repeated entries are excluded from the neighbors list.
That is, at iteration $t+1$ we can write ${\cal N}_{t+1} = {\cal N}_t \setminus {\cal T}_t$.
The objective function is then evaluated on all the neighbors  and the solution is updated according to
\begin{eqnarray}
    \label{eq:simpleTABU}
    \tilde{\bfx}_{t+1} =  \argmin_{\tilde{\bfx} \in {\cal N}_t} \tilde{\bfx}^{\top} \bfA \tilde{\bfx} + \tilde{\bfx}^{\top} \bfb.
\end{eqnarray}
The Tabu list is then updated, adding $\bfx_{t+1}$ to it
and, possibly omitting the $\bfx_{t-\ell}$ term from it.
The Tabu algorithm performs the described procedure for $T$ steps and returns the approximation vector with the lowest value obtained in the search.
Note that each step in TS requires ${\cal O}(k)$ function evaluations. While this approach does not guarantee a global solution, it tends to yield low approximation error.

\paragraph{Simulating Adiabatic Bifurcations (SAB).} A different approach that has been developed recently based on differential equations and molecular dynamics is SAB \cite{goto2019combinatorial}. In this approach, the decision variables in $\bfx$ are relaxed to be continuous, and a Hamiltonian function that has binary stationary points is minimized by simulating the dynamics to a steady state. This process is also known as annealing. In this approach, the cost of each step is a single matrix-vector product. However, as is typical with annealing methods, many steps are needed to reach convergence. 
As the problem size, $k$, grows, it becomes difficult to obtain the exact solution with either TS or SAB~\cite{oshiyama2022benchmark}.  Furthermore, as we show in our experiments in Section \ref{sec:experiments}, their computation time is high, making predictions in real-time impractical using these approaches.

\subsection{Graph Neural Networks}
\label{sec:relatedGNN}
Graph Neural Networks (GNNs) were introduced in \citet{scarselli2008graph} and became popular in  \cite{kipf2016semi}. GNNs were shown to be effective in various applications and fields. A non-exhaustive list includes computer vision \cite{wang2018dynamic}, social network analysis in semi-supervised settings \cite{kipf2016semi, velickovic2018graph}, bioninformatics \cite{jumper2021highly}. Essentially, GNNs allow the processing and learning from unstructured data described by graphs, to solve tasks such as node classification and graph classification, among other tasks. We refer the interested reader to \citet{waikhom2023survey} for a comprehensive review. Therefore, it was proposed in \citet{schuetz2022combinatorial}  to leverage GNNs as an optimization framework for the solution of relaxations of instances of QUBO problems. In these approaches, the idea is to train a GNN from scratch for every pair $(\bfA, \bfb)$.

Uniquely, we propose to use GNNs as a computational model to accelerate and \emph{generalize} the solution approximation of QUBO problems, by viewing the problem as a heterophilic node classification task. We elaborate on this perspective in Section \ref{sec:combToGraph}.  
Node classification problems are now a common application of GNNs  \cite{kipf2016semi, velickovic2018graph, Pei2020Geom-GCN:, chamberlain2021grand}. It is common to categorize node classification problems into \emph{homophilic} and \emph{heterophilic}, or non-homophilic types. To define homophily and heterophily, we follow the definition in \citet{Pei2020Geom-GCN:}, which measures the ground-truth label similarity of neighboring nodes. As shown in \citet{chamberlain2021grand}, diffusion-based GNNs are useful for homophilic graphs, where neighboring nodes have similar labels. However, heterophilic node classification is deemed as a more challenging problem, and has been extensively studied in recent years. For example, some methods propose to learn both low and high pass filters \cite{eliasof2023improving}, while other methods focus on taking inspiration from neural ordinary differential equations \cite{chen2018neural} to incorporate operators that allow to model heterophilic behavior. A recent example is GREAD \cite{choi2023gread} that use reaction and diffusion terms to generate non-smooth patterns, as well as methods like CDE \cite{zhao2023graph} that utilize convection-diffusion dynamics. Other approaches consider learning directional message passing neural networks, as in Dir-GNN \cite{dirgnn_rossi_2023}.
In the next section, we describe how a QUBO problem is transformed into a heterophilic node classification problem, and how to leverage the power of GNNs to solve such problems.

On another note, a recent line of works under the theme of Neural Algorithmic Reasoning \cite{velivckovic2021neural} suggests to train a neural network to replicate the exact behavior of polynomial complexity algorithms. In this work, however, we seek to accelerate the computation of approximated solutions to the NP-complete QUBO problem.

\section{From Combinatorial Optimization to Graph Learning Problems}
\label{sec:combToGraph}

In this section, we lay the foundation and assumptions of the proposed approach in this paper that shows
how a difficult optimization problem can be approximated by learning to solve a node classification problem. We then analyze the characteristics of the node classification problem to be solved, showing its connection with heterophilic graphs.

\subsection{Assumptions} 
\label{sec:assumptions}
Let $\bfx_{o}$ be the solution of the optimization problem in \cref{eq:bp}. Given, $\bfb$, computing the exact solution requires
solving an NP-complete optimization problem \cite{karp2010reducibility}. Thus, it is common to compromise on an approximation solution, that can be obtained by classical solvers as described in Section \ref{sec:related}. Still, those solvers require significant computational time and resources. We will now show that we are able to approximate the solutions using a neural approach, harnessing the power of GNNs.
\begin{proposition}
\label{prop:learnmap}
Given the observed vector $\bfb \in {\mathbb R}^k$, there is a piece-wise differentiable function $q:\mathbb{R}^k \rightarrow {\mathbb B}^k$ that maps the vector $\bfb$ to a binary solution vector $\bfx_{o} \in {\mathbb{B}}^k$. 
\end{proposition}

Let us assume  
that we have $n$ solved problems, for different observation vectors. Thus, we have observation-solution pairs
$\{ \bfb^{(j)}, \bfx_{{o}}^{(j)}\}$, $j=0,\ldots, n-1$.
 Then,  following Proposition \ref{prop:learnmap}, we suggest that 
we can approximate the function $q:\mathbb{R}^k \rightarrow {\mathbb B}^k$ via a neural network. If that was not the case, then by theory and experimental findings, learning $q$ would be a highly challenging task \citep{selmic2002neural, petersen2018optimal, sitzmann2020implicit}. Specifically, in the context of GNNs, the observed vector $\bfb$ and the solution $\bfx_{o}$ can be thought of as node features and labels, respectively. The graph connectivity, i.e., adjacency matrix, is defined by the sparsity of the matrix $\bfA$. Because our goal is to map every node to a binary value (0 or 1), the problem of learning the map $q$ can be formulated as a node classification problem, a common application of GNNs \cite{kipf2016semi}.

\subsection{Properties of the QUBO Problem}
\label{sec:BP_piece-wise}
As stated in Proposition \ref{prop:learnmap}, 
we seek to utilize a node classification GNN to approximate QUBO solutions. Therefore, and in light of our discussion about homophilic and heterophilic datasets in Section \ref{sec:relatedGNN}, it is important to understand the properties of the solutions that we wish to approximate. We now analyze the sensitivity of the solution $\bfx_{o}$ with respect to the observed vector $\bfb$.

\paragraph{The Landscape of the Map $q(\cdot)$.}
To justify the ability to learn a map, $q:\mathbb{R}^k \rightarrow {\mathbb B}^k$, from the observed vector $\bfb$ to the QUBO solution, $\bfx_o$, 
we need to understand the behavior of the function
$q(\bfb)$, and  its sensitivity to changes in $\bfb$. 
While such analysis is not known to us for the general case, there are specific problems with physical interpretation, as in \cref{ex:ising_sensitivityLearning}.

\begin{example}[The Ising model solution sensitivity is piece-wise constant]
\label{ex:ising_sensitivityLearning}
Consider the Ising model with constant coefficients \cite{duminil2016new,dziarmaga2005dynamics}, that reads:
\begin{eqnarray}
    \label{ising}
    I(\bfx, \bfA, \bfb) = \bfx^{\top} \bfA \bfx - b \bfx^{\top} \bfe,
\end{eqnarray}
where $\bfe$ is a vector of ones, the matrix $\bfA$ is binary and $b$ is a scalar.  
The solution, $\bfx_o$, remains fixed for a continuous range of $b$, and it changes rapidly at some $b_{c}$ that is called the critical value. That is, in this case of the Ising model, we have that $d\bfx_o/db$ is a piece-wise constant function.
\end{example}

To analyze the solution in a more general manner, we start by proving the following theorem:
\begin{theorem}
   \label{thm:BP_differentiable}
   The QUBO problem in \cref{eq:bp}, at its optimum, i.e., $f(\bfx_o)$, is piece-wise differentiable with respect to changes of the observed vector $\bfb$. 
\end{theorem}
\begin{proof}
Let us define the function
\begin{equation}
\label{eq:fp}
f(\bfx, \epsilon) = 
\bfb^{\top} \bfx + \bfx^{\top} \bfA \bfx + \epsilon \bfdelta^{\top} \bfx.
\end{equation}
where $\|\bfdelta\|=1$ is a direction vector.
Let $\bfx_{\epsilon}$ be the minimizer of \cref{eq:fp} and let $f_{\epsilon} = f(\bfx_{\epsilon}, \epsilon)$. Then, by the definition of the minimizer, we have that for any $\epsilon$, the following holds:
\begin{equation}
\label{ineq1}
    f_0 + \epsilon \bfdelta^{\top} \bfx_{\epsilon} \le f(\bfx_{\epsilon},0) + \epsilon \bfdelta^{\top} \bfx_{\epsilon}  = f(\bfx_{\epsilon},\epsilon) = f_{\epsilon}
\end{equation}
Similarly, we have that:
\begin{equation}
\label{ineq2}
    f_{\epsilon} = f(\bfx_{\epsilon},\epsilon)  \le f(\bfx_0,\epsilon) = f(\bfx_0,0) + \epsilon \bfdelta^{\top} \bfx_{0} = f_0 + \epsilon \bfdelta^{\top} \bfx_{0}.
\end{equation}
This means that we can bound $f_{\epsilon}$ as follows:
\begin{equation}
\label{ineq3}
f_0 + \epsilon \bfdelta^{\top} \bfx_{\epsilon} \le f_{\epsilon} \le f_0 + \epsilon \bfdelta^{\top} \bfx_{0}.
\end{equation}
Subtracting $f_0$ from both sides and dividing by $\epsilon$ and assuming $0<\epsilon$ we obtain that
\begin{equation}
\label{eq:ineq4}
 \bfdelta^{\top} \bfx_{\epsilon} \le {\frac {f_{\epsilon} - f_0}{\epsilon}} \le   \bfdelta^{\top} \bfx_{0}.
\end{equation}
Because the norms of $\bfx_0$ and $\bfx_{\epsilon}$ are bounded (as they are binary) we obtain upper and lower bounds for the change in the optimal solution in the direction $\bfdelta$. This is the definition of Clarke derivative \cite{michel1992generalized}, which extends the definition of the derivative beyond a point to an interval.
The function $f(\bfx_o)$ is therefore piece-wise linear. 
\end{proof}

Theorem \ref{thm:BP_differentiable} shows that the QUBO problem in \cref{eq:bp} is piece-wise \emph{differentiable} with respect to $\bfb$, a property that we now utilize in order to characterize the QUBO solution behavior with respect changes in the observed vector $\bfb$.

\begin{observation} The QUBO $\bfx$ solution is piece-wise constant with respect to changes in $\bfb$.
    \label{obs:bppiece-wise}
\end{observation}
\begin{proof}
To study the effect of a change in $\bfb$ on the solution $\bfx_{\epsilon}$ we now choose
$\bfdelta = \bfe_j$, that is, the unit vector in direction $j$. This simplifies \Cref{eq:ineq4} to
\begin{eqnarray}
\label{ineq5}
 \epsilon   \bfx_{\epsilon_j} \le  \epsilon   \bfx_{0_j},
\end{eqnarray}
where the $j$ stands for the $j$-th entry in the vector.
Recall that $\bfx$ gets only binary values. If, $\epsilon>0$, $\bfx_{0_j} = 0$ and $\bfx_{0_j} \not= \bfx_{\epsilon_j}$ then we must have that
$\bfx_{\epsilon_j} = 1$, which yields a contradiction. This implies that 
$\bfx_{0_j} = \bfx_{\epsilon_j}, \forall \epsilon > 0$.

Furthermore, in this case $f_0 = f_\epsilon$ which implies that
$\bfx_0 = \bfx_\epsilon$ is a solution of the perturbed problem. Thus, we obtain that
\begin{equation}
\label{eq:eq10}
    {\frac {\delta \bfx_{0_{j}}}{\epsilon^+}} = 0 \quad {\rm if}\ \ \ \bfx_{0_j} = 0. 
\end{equation}
Similarly, if $\bfx_{0_j} = 1$ and $\epsilon < 0$
and $\bfx_{\epsilon_j} \not=\bfx_{\epsilon_j}$, then $\bfx_{\epsilon_j}=0$.
But this again yields a contradiction, which means
\begin{equation}
\label{eq:eq11}
    {\frac {\delta \bfx_{0_j}}{\epsilon^-}} = 0 \quad {\rm if}\ \ \ \bfx_{0_j} = 1.
\end{equation}
Because Equations (\ref{eq:eq10})--(\ref{eq:eq11}) show that the change in $\bfx$ is zero when changing b in $\epsilon^+$ and $\epsilon^-$. This implies that $\bfx$ remains constant for a wide range of perturbations in $\bfb$. Thus, by definition, $\bfx$ is piecewise constant with respect to $\bfb$.
 \end{proof}
 Below, in \cref{ex:landscape}, we provide a concrete example of this property, and showcase its piece-wise constant landscape behavior with respect to the vector $\bfb$.
\begin{example}
\label{ex:landscape}
   Let us randomly choose $\bfA \in \mathbb{R}^{20 \times 20}$ and $\bfb \in \mathbb{R}^{20}$. We compute the exact solution $\bfx_o$ using exhaustive search. We then choose two random orthogonal directions $\bfb_1$ and $\bfb_2$. Using these fixed vectors we find the solution
\begin{eqnarray}
\label{eq:stcut}
  \bfx(s,t) = {\rm argmin}_{\bfx} \left(\bfb + t \bfb_1 + s \bfb_2\right)^{\top} \bfx + \bfx^{\top} \bfA \bfx.
\end{eqnarray}
For each $\bfx(s,t)$ we compute the distance
\begin{eqnarray}
\label{eq:stcut_dist}
  \phi(s,t) = \| \bfx(s,t) - \bfx(0,0) \|^2. 
\end{eqnarray}
In Figure~\ref{fig:sens}, we plot $\phi$ as a function of $s$ and $t$ for four different experiments where the matrix, $\bfA$, the vector, $\bfb$ and the perturbations $\bfb_1$ and $\bfb_2$ are randomly chosen.
In all cases, we see that the solution is constant over a range of values and then sharply changes. Such surfaces are good candidates for deep network approximation. 
\end{example}

\begin{figure*}[t]
\centering
\begin{minipage}{.4\textwidth}
  \centering
    \centering
    \begin{tabular}{cc}
    \includegraphics[width=3.0cm]{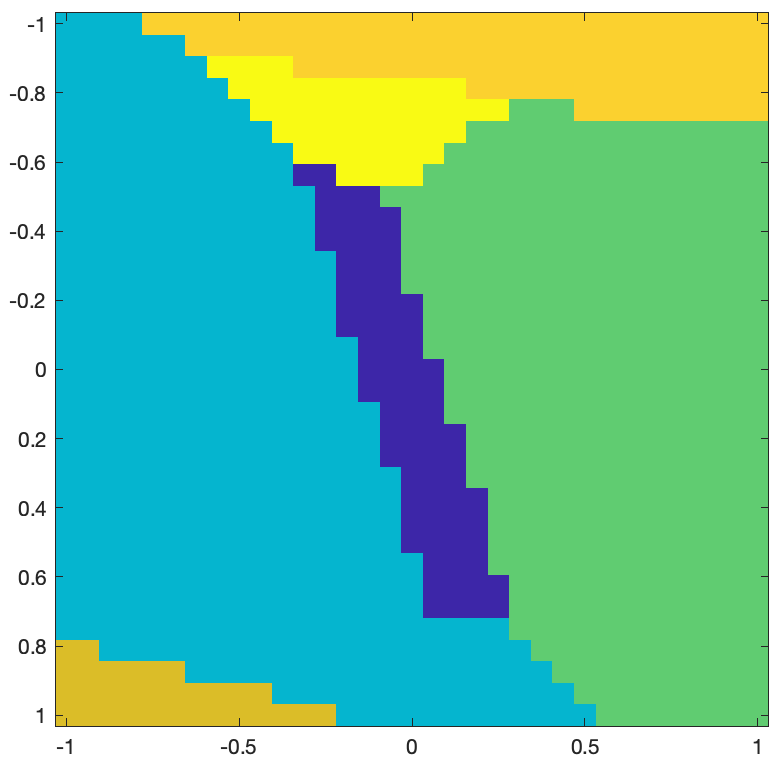}     &  \includegraphics[width=3.0cm]{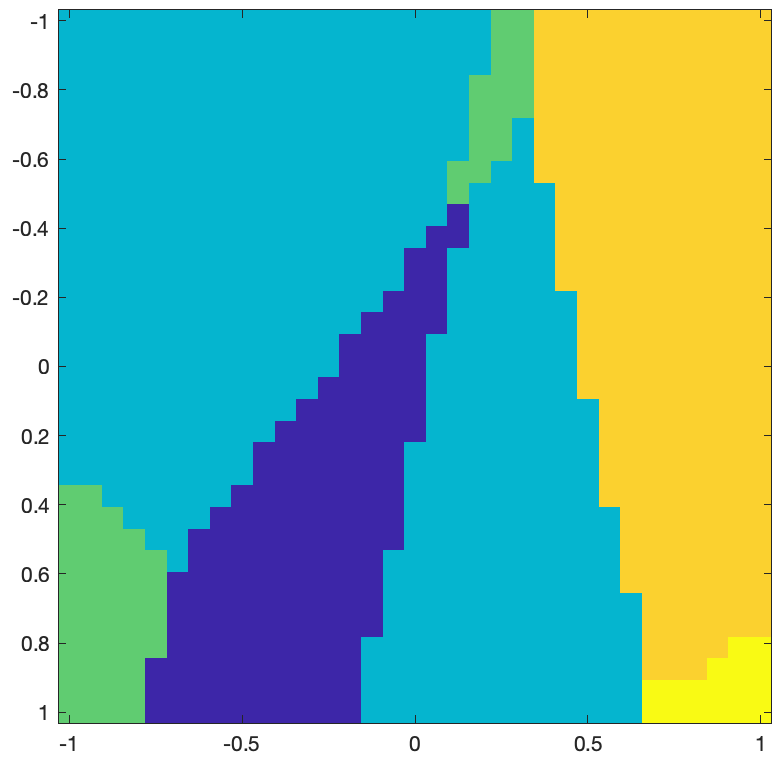}\\
    \includegraphics[width=3.0cm]{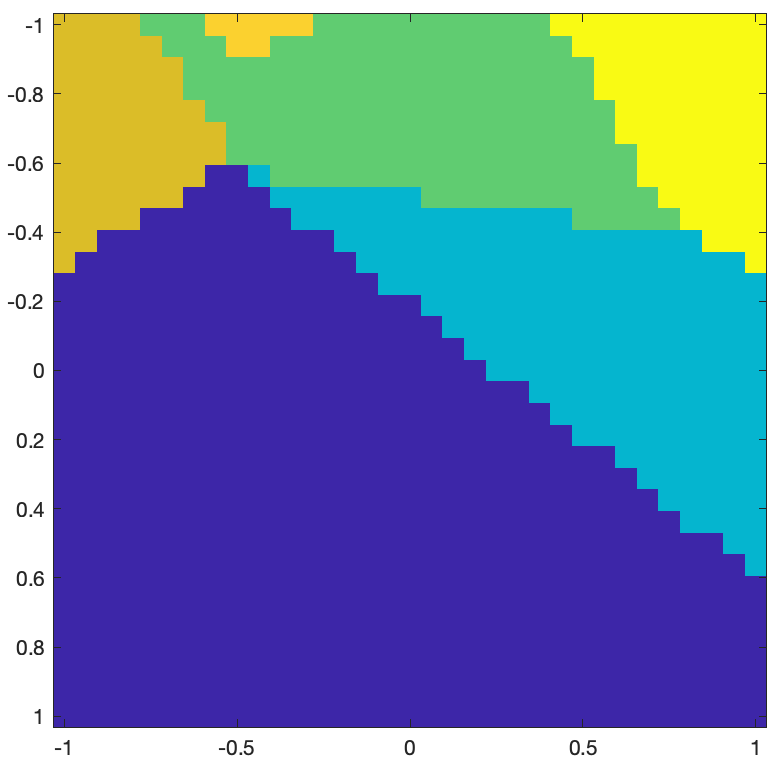}     & \includegraphics[width=3.0cm]{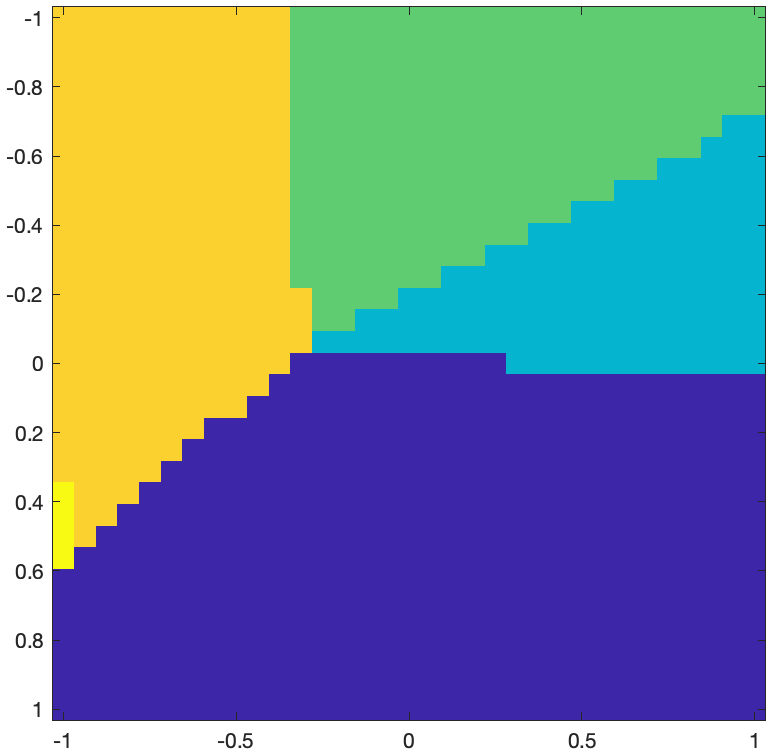}
    \end{tabular}
    \caption{The function $\phi(s,t)$ in \cref{eq:stcut_dist} for four different random realizations of $\bfA, \bfb, \bfb_1$ and $\bfb_2$.  $\phi(s,t)$ is piece-wise constant, favored by neural networks.}
    \label{fig:sens}
\end{minipage}%
\hspace{3em}
\begin{minipage}{.4\textwidth}
    \centering
    \begin{tabular}{cc}
    \includegraphics[width=2.97cm]{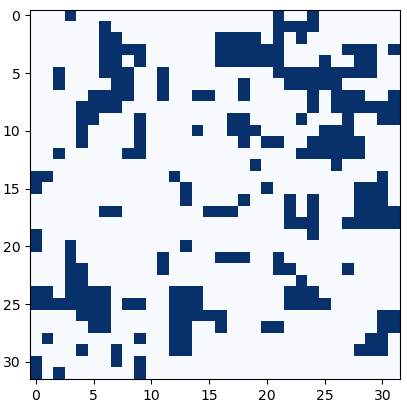}     &  \includegraphics[width=2.97cm]{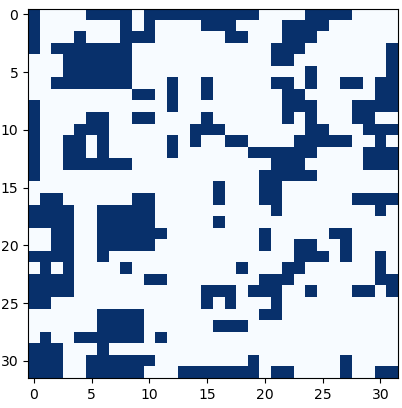}\\
    \includegraphics[width=2.97cm]{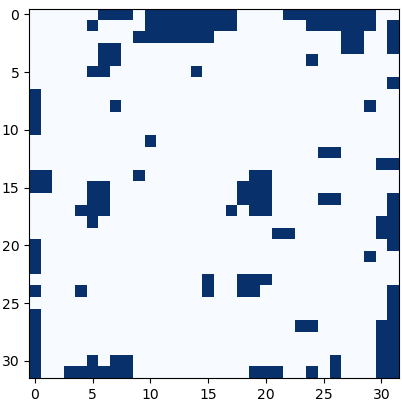}     & \includegraphics[width=2.97cm]{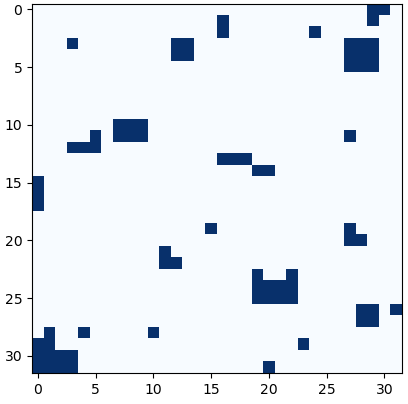}
    \end{tabular}
    \caption{The solution $\bfx_{o}$ for different random values of the vector $\bfb$. The graph in this example is a regular lattice where each node (pixel) $i,j$ is connected to its 4 neighbors.}
    \label{fig:sols}
\end{minipage}

\vspace{1em}
\end{figure*}

\section{Approximating QUBO Solutions with GNNs}
\label{sec:GNN}

Section \ref{sec:combToGraph} discusses the ability to learn the map $\bfx_o = q(\bfb)$, as well as the characteristics of this map. 
 In this section, we explore the behavior of the solution vector  $\bfx$, which hints about the needed type of map. We then propose an appropriate architecture to approximate the map.
\subsection{The Heterophilic Behavior of QUBO Solutions}
 \label{sec:heterophilic_solution}
While it is difficult to predict the exact distribution of the solution $\bfx$, a few simple cases have been studied in the context of the Ising model \cite{dziarmaga2005dynamics}. For this model, it is known that when the solution resides on a regular lattice, aside from trivial cases discussed in Appendix \ref{app:extremes}, it can exhibit high-frequency patterns. In graphs and GNNs, such solutions are commonly affiliated with heterophily \cite{Pei2020Geom-GCN:, choi2023gread}.
To demonstrate that, we solve four different problems on a lattice of size $32 \times 32$ using the SAB algorithm  \cite{goto2019combinatorial}
and plot the solutions in Figure \ref{fig:sols}. 

It is observed that the obtained solutions are irregular and contain small connected components of uniform values. This implies that typical smoothing-based GNNs 
may face difficulties in expressing such non-smooth patterns. Therefore, as we discuss in Section \ref{sec:QUBO-GNN}, we leverage the power of heterophily-designated GNNs, combined with a specialized QUBO-aware component.  

\subsection{QUBO Graph Neural Network}
\label{sec:QUBO-GNN}

In Section \ref{sec:heterophilic_solution} and Figure \ref{fig:sols}, we established that the spatial behavior of the solution of the QUBO problem, $\bfx_0$, is heterophilic.  In this section, we design a GNN that is capable of generating spatially non-smooth patterns while being aware of its application -- predicting approximations of QUBO solutions. 

\paragraph{GNN Backbone.} Recently, several heterophily-designated GNNs have been proposed. A non-exhaustive list includes \cite{choi2023gread,eliasof2023improving,eliasof2024RD,dirgnn_rossi_2023}. 
In this paper, we choose to use a reaction-diffusion GNN, which is similar to  GREAD \cite{choi2023gread}, discussed in Section \ref{app:architecture}. Our choice is inspired by Turing patterns \cite{Turing52}, which are complex and have relation to the Ising model \cite{turingIsing}, which is an instance of a QUBO problem, as discussed in Example \ref{ex:ising_sensitivityLearning}.  Note that, while other heterophilic designated architectures can be used, this is not the focus of our work. Rather, our contribution is drawing the link between GNNs, node classification, and QUBO. 

\paragraph{QUBO Objective Awareness.} While the chosen reaction-diffusion GNN is powerful, we show that it is important also to allow the network to be aware of the objective we seek to minimize, that is, the QUBO function, as shown in \cref{eq:bp}. 
To this end, note that the QUBO objective can be written as the sum of the nodal residual of the vector $\bfr = \bfx \odot (\bfA \bfx + \bfb)$, where $\odot$ is the Hadamard product, which is added to the features as follows:
\begin{equation}
    \label{eq:quboAware}
     \bfr^{(\ell)}(\bfh^{(\ell)}, \bfA, \bfb) = \bfh^{(\ell)} \odot (\bfA \bfh^{(\ell)} + \bfb),
\end{equation}
This modification allows the network to implicitly track the nodal QUBO loss, allowing the network to be aware of the optimization problem at hand. As we show in our experiments in  Section \ref{sec:experiments}, this addition improves the solution approximation and training convergence.

\paragraph{QUBO-GNN.} Given input features $\bfb$, we first embed them to obtain the initial features using a Multilayer Perceptron (MLP) encoder denoted by ${\rm{enc}}$, as follows:
\begin{equation}
    \label{eq:embedding}
    \bfh^{(0)} = {\rm{enc}}(\bfb) \in \mathbb{R}^{k \times d}.
\end{equation}
We denote a GNN backbone layer by $\rm{GNN^{(\ell)}}$, which takes as input the node features at the $\ell$-th layer, $\bfh^{(\ell)}$, and outputs updated node features $\bfh^{(l+1)}$, i.e.,:
\begin{equation}
    \label{eq:basicGNN}
    \bfh^{(l+1)} = {\rm{GNN}^{(\ell)}}(\bfh^{(\ell)}).
\end{equation}
As described in \cref{eq:quboAware}, we include QUBO-aware features. Thus, a QUBO-GNN layer reads:
\begin{equation}
    \label{eq:bpgnn}
        \bfh^{(l+1)} = {\rm{GNN}^{(\ell)}}(\bfh^{(\ell)}+g^{(\ell)}(\bfr^{(\ell)})),
\end{equation}
where $g$ is an MLP.
As discussed earlier, for our $\rm{GNN}$ implementation, we use a reaction-diffusion GNN, discussed In Section \ref{app:architecture}.
Since we consider a binary problem, after $L$ layers, we use a decoder (classifier) denoted ${\rm{dec}}$,   implemented by a linear layer to obtain the predicted  solution:
\begin{eqnarray}
\label{yFromu}
{\bfx_{p}} = {\rm{dec}}(\bfh^{(\ell)})\in \mathbb{R}^{k \times 1}.
\end{eqnarray}

\subsection{QUBO-GNN Architecture}
\label{app:architecture}

We now provide details on the implementation of our QUBO-GNN, which combines concepts of reaction-diffusion similar to \cite{ choi2023gread} and our proposed QUBO aware node features.

Overall, our QUBO-GNN is defined by the following ordinary differential equations (ODE) and initial conditions:
\label{eq:discDR}
\begin{eqnarray}
\label{eq:discDRa}
&& {\dot \bfh} =  -\bfL (\bfh + \bfr) \bfSigma(t) + f(\bfh, t; \bftheta(t)) \\
&& \bfr(\bfh, \bfA, \bfb; t) = q(\bfh \odot (\bfA \bfh + \bfb); t) \\
\label{eq:discDRb}
&& \bfh(t=0) = \bfh_0 = {\rm{enc}}(\bfb) 
\end{eqnarray}
Here, $\bfh(t)\in\mathbb{R}^{k \times d}$ is a hidden feature matrix with $d$ features that depends on the initial state $\bfh_0$, and the input node features $\bfb\in \mathbb{R}^{k \times 1}$. ${\dot \bfh = \frac{d\bfh(t)}{dt}}$ denotes the first-order derivative with respect to time. $\bfL$ is the graph Laplacian as in \cite{eliasof2021pde}. $\bfr$ is a continuous version of the QUBO objective aware features as shown in \cref{eq:quboAware}, and $q$ is an MLP. The initial state is obtained by embedding the input node features using a single layer MLP; $\bfh_{0} ={\rm{enc}}(\bfb)$. Also,
$\bfSigma(t) \geq 0$ is a diagonal matrix with non-negative learnable diffusion coefficients on its diagonal.
The reaction term  $f$ from Equation (\ref{eq:discDRa}) is a learnable point-wise non-linear function that depends on the current hidden state $\bfh(t)$. 

In order to obtain a graph neural layer from  Equation (\ref{eq:discDRa})-(\ref{eq:discDRb}), we use the forward Euler discretization method, as is common in GNNs \cite{chamberlain2021grand,eliasof2021pde}. Therefore, a QUBO-GNN layer is defined by the following update rules shown in Algorithm \ref{alg:architecture}.

\paragraph{Implementation.} Our code is implemented in Pytorch \cite{pytorch} and Pytorch-Geometric \cite{pyg2019}, and will be published together with the datasets upon acceptance. All the experiments were conducted on an Nvidia RTX-4090 GPU with 24GB of memory. 

\begin{algorithm}
\footnotesize
\caption{QUBO-GNN Layer}\label{alg:architecture}
\textbf{Input:} Node features $\bfh^{(\ell)} \in \mathbb{R}^{k\times d}$\\
\textbf{Output:} Updated node features $\bfh^{(l+1)} \in \mathbb{R}^{k \times d}$ 
\begin{algorithmic}[1]
\STATE Compute QUBO Aware node features: \\
$ \bfr^{(\ell)} = \bfh^{(\ell)} \odot (\bfA \bfh^{(\ell)} + \bfb) $
   \STATE Diffusion: \\
   $\bfh^{(l+1/2)} = \bfh^{(\ell)} - \epsilon \bfL(\bfh^{(\ell)}+g^{(\ell)}(\bfr^{(\ell)}))$.
   \STATE Reaction: \\ $\bfh^{(l+1)} = \bfh^{(l+1/2)} + \epsilon f(\bfh^{(l+1/2)}; \bftheta_{r}^{(\ell)}).$
   \STATE Return $\bfh^{(l+1)}$.
 \end{algorithmic}

\end{algorithm}

\section{Generating Training Data}
\label{sec:training}
In order to train the network, given a matrix $\bfA$, we require to have observation-solution pairs
$\{ \bfb^{(j)}, \bfx_{{o}}^{(j)}\}$, $j=0,\ldots, n-1$. Given these pairs, we use the binary cross-entropy loss 
function between $\bfx_p^{(j)}$, the predicted solution obtained from QUBO-GNN, and the ground-truth solution $\bfx_o^{(j)}$. 

\paragraph{Data Acquisition is Computationally Expensive.} An important aspect is how data pairs observation-solution pairs
$\{ \bfb^{(j)}, \bfx_{{o}}^{(j)}\}$ can be obtained. Because, as discussed throughout the paper, given an observation vector $\bfb^{(j)}$,  finding the optimal solution $\bfx_{{o}}^{(j)}$ is an NP-complete problem.
For small problems (e.g., $k \le 20$), one can use exhaustive search techniques.
This approach requires solving the optimization problem in order to obtain a training set. Clearly, such an approach cannot be used for large-scale problems. For example, if $k=1024$, one has to check $2^{1024}$ possible solutions.

\paragraph{Practical Data Pairs Generation.} Due to the computational challenge described above, we derive a novel data generation technique. Instead of choosing the observed vector $\bfb^{(j)}$ and then solving an NP-complete problem to find $\bfx_{{o}}^{(j)}$, it is well-known that it is much easier to find a vector $\bfb$, given the pair $(\bfA,\bfx)$ such that $\bfx$ is the minimizer of the QUBO problem \citep{conforti2014integer}. Thus, we choose $\bfx_{{o}}^{(j)}$ first, and only then find a corresponding $\bfb^{(j)}$. This procedure allows to create data pairs while avoiding the solution of an NP-complete problem, and we consider a relaxed optimization problem, where $\bfx$ is {\em continuous} in the interval $[0,1]$:
\begin{eqnarray}
    \label{optc}
    \min_{\bfx \in [0,1]} \bfx^{\top} \bfA \bfx + \bfx^{\top} \bfb. 
\end{eqnarray}
One popular way to solve this minimization problem is by the \emph{log barrier} approach \cite{nw}, which is obtained by adding the logarithmic terms  as follows:
\begin{eqnarray}
    \label{opt}
    \min \bfx^{\top} \bfA \bfx + \bfx^{\top} \bfb - \mu \sum \left(\log  \bfx + \log(1-\bfx)\right).
\end{eqnarray}
As shown in \citet{nw}, the conditions for a minimum are:
\begin{equation}    
 (\bfA + \bfA^{\top}) \bfx + \bfb - {\frac \mu {\bfx}} +
{\frac \mu {1-\bfx}} = 0.\end{equation}
Now, suppose that we pick $\bfx_o$ randomly to be $\{\epsilon, 1-\epsilon\}^k$, where $\eps$ is small and define $\bfb$ as follows:
\begin{equation}
\label{eq:addSigma}
   \bfb_o =  -(\bfA + \bfA^{\top}) \bfx_o  + {\frac \mu {\bfx_o}} -
{\frac \mu {1-\bfx_o}}, \quad 
\bfb = \bfb_o + \sigma^2 \bfz,
\end{equation}
where $\bfz \sim N(0, \bfI)$ is a random vector.
The vector $\bfb_o$ corresponds to an approximate solution to the problem in \cref{optc}, which is also the solution of the binary problem in \cref{eq:bp}.
By adding a noise term $\sigma^2 \bfz$ to the observed vector $\bfb_o$ in \cref{eq:addSigma},
the solution of the continuous problem
is now different from the solution of the discrete binary problem.
As explored in the previous section, the solution of the discrete problem varies little when $\sigma$ is not large. We therefore use a small number of Tabu search steps, starting from $\bfx_o$ to modify the solution if needed. While there is no guarantee that the solution is minimal, we have verified by comparison with exhaustive search solutions on small-scale problems that, indeed, we obtain the global minima every time. This approach, of choosing $\bfx_o$ rather than $\bfb$ allows us to quickly generate large amounts of data even for large-scale problems (with hundreds or thousands of decision variables), where finding the exact solution is computationally intractable. Our procedure for data generation is summarized in Algorithm \ref{alg:getdata}.
\begin{algorithm}[t]
\footnotesize
   \caption{Practical Data Generation}
   \label{alg:getdata}
       {\bfseries Input:} Problem size $k$, Variance $\sigma$, Matrix $\bfA$, parameter $\mu=10^{-3}$. \\
       {\bfseries Output:} Data pair $(\bfb, \bfx)$. 
\begin{algorithmic}[1]
   \STATE Randomly choose $\bfx_o \in \{\epsilon, 1-\epsilon\}^k$.
   \STATE Set $\bfb_{o} =  -(\bfA + \bfA^{\top}) \bfx_o  + {\frac \mu {\bfx_o}} -
{\frac \mu {1-\bfx_o}} $.
   \STATE Perturb $\bfb \leftarrow \bfb_o + \sigma^2 \bfz$, where $\bfz\sim N(0, \bfI)$
   \STATE Use up to ten Tabu search iterations to maintain efficiency, and solve \cref{eq:bp} starting from $\bfx_o$ to obtain improved an solution $\bfx$.
   \STATE Return $(\bfb,\bfx)$.
 \end{algorithmic}
\end{algorithm}

\section{Numerical Experiments}
\label{sec:experiments}

\begin{table*}[t]
\footnotesize
  \caption{A comparison between a naive exhaustive search method, classical methods, GNNs, and our QUBO-GNN. Inference time is denoted by IT.}
  \setlength{\tabcolsep}{2pt}
  \centering
  \begin{tabular}{lcccccccccccccc}
    \toprule Size $\rightarrow$
     & \multicolumn{3}{c}{Small ($k=10$)} & \multicolumn{3}{c}{Medium ($k=256$)} & \multicolumn{3}{c}{Large $(k=1024)$} \\ 
    \cmidrule(r){2-4} \cmidrule(r){5-7} \cmidrule(r){8-10}
    & \multirow{2}*{Acc(\%)} & Rel. QUBO & \multirow{2}*{IT(ms)} & \multirow{2}*{Acc(\%)} & Rel. QUBO & \multirow{2}*{IT(ms)} & \multirow{2}*{Acc(\%)} & Rel. QUBO & \multirow{2}*{IT(ms)} \\
   {Method $\downarrow$} &  & Objective &  & & Objective &  & & Objective &  \\
    \midrule
    \textbf{\textsc{Naive.}} \\
    Exhaustive & 100 & 0 & 8 & 100 & 0 & $10^7$days & 100 & 0 & $10^{302}$days \\
    \midrule
    \textbf{\textsc{Classical.}} \\
    TS & 98.2$_{\pm0.1}$ & 5.90$\cdot$$10^{-3}$ & 439 & 96.4$_{\pm0.3}$ & 4.72$\cdot10^{-3}$ & 2197 & 93.9$_{\pm0.2}$ & 4.58$\cdot10^{-3}$ & 4409 \\
    SAB & 100 & 0 & 298 & 91.4$_{\pm0.2}$ & 5.38$\cdot$$10^{-2}$ & 5340 & 89.2$_{\pm0.4}$ & 7.96$\cdot$$10^{-2}$ & 40007 \\
    \midrule
    \textbf{\textsc{GNN.}} \\
    GCN & 70.7${_{\pm2.4}}$ & 7.12$\cdot$$10^{-1}$ & 2.8 & 69.4$_{\pm2.8}$ & 5.03$\cdot$$10^{-1}$ & 2.9 & 62.7$_{\pm2.8}$ & 8.76$\cdot$$10^{-1}$ & 9.0 \\
    GAT & 71.3$_{\pm1.3}$ & 6.98$\cdot10^{-1}$ & 3.0 & 79.7$_{\pm1.4}$ & 4.38$\cdot10^{-1}$ & 3.2 & 65.5$_{\pm1.7}$ & 7.23$\cdot10^{-1}$ & 12.9 \\
    GIN & 74.4$_{\pm2.8}$ & 6.02$\cdot10^{-1}$ & 2.8 & 73.1$_{\pm1.5}$ & 4.19$\cdot 10^{-1}$ & 2.9 & 71.6$_{\pm1.7}$ & 6.20$\cdot 10^{-1}$ & 9.2 \\
    GCNII & 75.1$_{\pm3.4}$ & 5.67$\cdot10^{-1}$ & 2.6 & 87.1$_{\pm1.9}$ & 4.13$\cdot10^{-1}$ & 2.9 & 81.6$_{\pm2.1}$ & 5.08$\cdot10^{-1}$ & 9.1 \\
    Dir-GNN & 89.6$_{\pm0.6}$ & 3.74$\cdot10^{-2}$ & 2.8 & 90.0$_{\pm0.7}$ & 1.07$\cdot$$10^{-1}$ & 3.1 & 89.2$_{\pm0.4}$ & 9.62$\cdot$$10^{-2}$ & 10.3 \\
    \midrule
    \textbf{\textsc{Ours.}} \\
    QUBO-GNN & 98.2$_{\pm0.5}$ & 2.83$\cdot$$10^{-3}$ & 2.8 & 97.2$_{\pm0.4}$ & 5.31$\cdot$$10^{-3}$ & 3.0 & 93.7$_{\pm0.5}$ & 1.26$\cdot$$10^{-2}$ & 10.1 \\
    QUBO-GNN+TS & 98.4$_{\pm0.3}$ & 1.37$\cdot$$10^{-5}$ & 17 & 97.4$_{\pm0.3}$ & 1.21$\cdot$$10^{-5}$ & 194 & 94.5$_{\pm0.2}$ & 1.10$\cdot$$10^{-4}$ & 746 \\
    \bottomrule
  \end{tabular}
  \label{tab:main_results}
  \vspace{1em}
\end{table*}

In this section, we evaluate QUBO-GNN and other baselines to understand whether our theoretical observations on the behavior of the , characterized in Section \ref{sec:BP_piece-wise} and Section  \ref{sec:heterophilic_solution} -- hold in practice. 

\paragraph{Baselines.} We consider: (i) naive methods, i.e., exhaustive search, (ii) classical solvers such as TS \cite{glover1998tabu} and SAB \cite{goto2019combinatorial}, and (iii) GNN baselines, namely GCN \cite{kipf2016semi}, GAT\cite{velickovic2018graph}, GIN \citep{xu2018how},  GCNII \cite{chen20simple}, and Dir-GNN \cite{dirgnn_rossi_2023}. We chose to compare our method with those baselines as they are representative of node classification architectures from recent years. Also, Dir-GNN is heterophily designated, thus suitable to be tested with the heterophilic behavior of QUBO solutions, as discussed in Section  \ref{sec:heterophilic_solution}.

\paragraph{Metrics.} In our experiments, we consider three metrics: 

\begin{enumerate}
\item \textbf{Solution Accuracy (\%).} Accuracy is the fraction of correctly labeled nodes per example. Given a graph $G=(V,E)$ with $n$ nodes, predicted solution $\bfx_p$, and ground truth $\bfx_o$, it is defined as
\begin{equation}
    \label{eq:accuracy}
    {\rm Accuracy} = \frac{\sum_{i=0}^{n-1} \mathbf{1}_{\bfx_{o_i}=\bfx_{p_i}}}{n},
\end{equation}
where $\bfx_{o_i}$ and $\bfx_{p_i}$ are the $i$-th entries of $\bfx_o$ and $\bfx_p$, and $\mathbf{1}_{\bfx_{o_i}=\bfx_{p_i}}$ is 1 if equal and 0 otherwise. As discussed in Section \ref{sec:training}, the data-generation algorithm ensures, with very high probability \citep{conforti2014integer}, that $\bfx_o^{(j)}$ is the minimizer of the QUBO problem defined by $(\bfA,\bfb^{(j)})$. Hence, measuring accuracy in Section \ref{sec:experiments} on this dataset is reliable with high probability.

\item 
\textbf{Relative QUBO Objective.} Recall the QUBO objective from \cref{eq:bp}. At the optimal solution $\bfx_o$, its value reads:
\begin{equation}
      f_0 = f(\bfx_o; \bfb, \bfA) =  \bfx_{o}^{\top} \bfA \bfx_o + \bfx_{o}^{\top} \bfb.
\end{equation}
Similarly, we measure the QUBO objective at the predicted solution $\bfx_p$:
\begin{equation}
      f_p = f(\bfx_p; \bfb, \bfA) =  \bfx_{p}^{\top} \bfA \bfx_p + \bfx_{p}^{\top} \bfb.
\end{equation}
We then define the \emph{relative} QUBO objective as follows:
\begin{equation}
    \label{eq:relativeQubo}
    {\rm{RelQUBO}} = \frac{f_p - f_o}{|f_o|}
\end{equation}
Note that since $f_o$ is the optimal value in this minimization problem, it is always true that $f_p \geq f_o$, and therefore the lowest and best value of this metric is 0.

\item \textbf{Inference time.} We measure the time (ms) each approach (exhaustive search, classical solvers, other GNNs, and ours) takes to solve a QUBO, averaged over 100 problems of equal size per dataset on an Nvidia RTX-4090 GPU. For GNNs, unlike classical solvers \citep{schuetz2022combinatorial}, and following our approach of learning from QUBO examples to generalize solutions rather than solving from scratch, inference time excludes training, as standard in supervised neural networks. Thus, during inference, the network performs no additional optimization, but directly maps input $(\bfA,\bfb)$ to an approximate solution $\bfx$.

\end{enumerate}

\paragraph{Datasets.} Our experiments include three datasets generated by our procedure described in Section \ref{sec:training}, in varying sizes - small ($k=10$), medium ($k=256$) and large ($k=1024$). Note that the settings of $k=\{256,1024\}$ are considered highly challenging, as there are $2^{\{256,1014\}}$ possible solutions to the QUBO problem. We provide further details on the datasets in Appendix \ref{app:datasets}.

\paragraph{Training and Inference.}
Our training procedure utilizes the standard cross-entropy loss where the goal is to correctly classify nodes with the binary label of the QUBO solution $\bfx$.
For inference, we consider two modes to obtain an approximation of the solutions of QUBO problems: (1) \textbf{Neural.} As standard in neural networks, we employ QUBO-GNN to approximate the QUBO solution, and (2) \textbf{Hybrid.} We obtain a solution approximation using QUBO-GNN. This solution is considered an initial point in a classical solver. In our experiments, we used a small {\emph{fixed}} number (up to 10) iterations of TS \cite{glover1998tabu}, as discussed in Section \ref{sec:relatedBP}.

\paragraph{QUBO Solution Approximation.}
We compare the approximated solution using various baselines in  Table \ref{tab:main_results}, where we consider small, medium, and large problems. As discussed in \ref{sec:problemformulation}, we consider the case where the adjacency matrix $\bfA$ from \cref{eq:bp} is constant, while the observed vector $\bfb$ is variable. To ensure statistical significance, for each size, we generate 10 instances of the matrix $\bfA$ and report the average and standard deviation metrics in Table \ref{tab:main_results} and Appendix \ref{app:standard_deviations}. 

Our results indicate that: (i) Exhaustive search yields accurate solutions but is impractical for medium or large problems as the number of decision variables or nodes grows. (ii) Classical solvers achieve high accuracy on small problems, but their accuracy and efficiency decline as problem size increases. (iii) GNNs can generalize the QUBO problem solution by learning the map $q(\cdot)$, reducing inference time. Heterophily-designated GNNs like Dir-GNN and our QUBO-GNN offer higher accuracy and lower QUBO objectives. (iv) Combining neural approaches, such as our QUBO-GNN, with classical solvers like TS improves performance at a lower cost compared to classical solvers alone. Thus, our approach can serve as a strong prior for initializing classical solvers.

\begin{table}[t]
\footnotesize
    \centering
        \caption{The impact of adding the QUBO aware features from \cref{eq:quboAware} to the node features.
    }
    \begin{tabular}{ccccc}
    \toprule
        \multirow{3}*{{QUBO Features}}&  \multicolumn{3}{c}{{Large $(k=1024)$}}\\ 
       \cmidrule(r){2-4} 
      &  \multirow{2}*{Acc(\%)} & Rel. QUBO & Inference \\ & & Objective &  Time (ms)  \\   \midrule
        \xmark &  91.7$_{\pm0.8}$ & 6.18$\cdot10^{-2}$ &  9.3 \\ 
        \cmark & 93.7$_{\pm0.5}$ & 1.26$\cdot 10^{-2}$ & 10.1\\
         \bottomrule
    \end{tabular}
\label{tab:ablation_quboLoss}
\end{table}

\begin{figure}[t]
    \centering
    \includegraphics[width=0.7\linewidth]{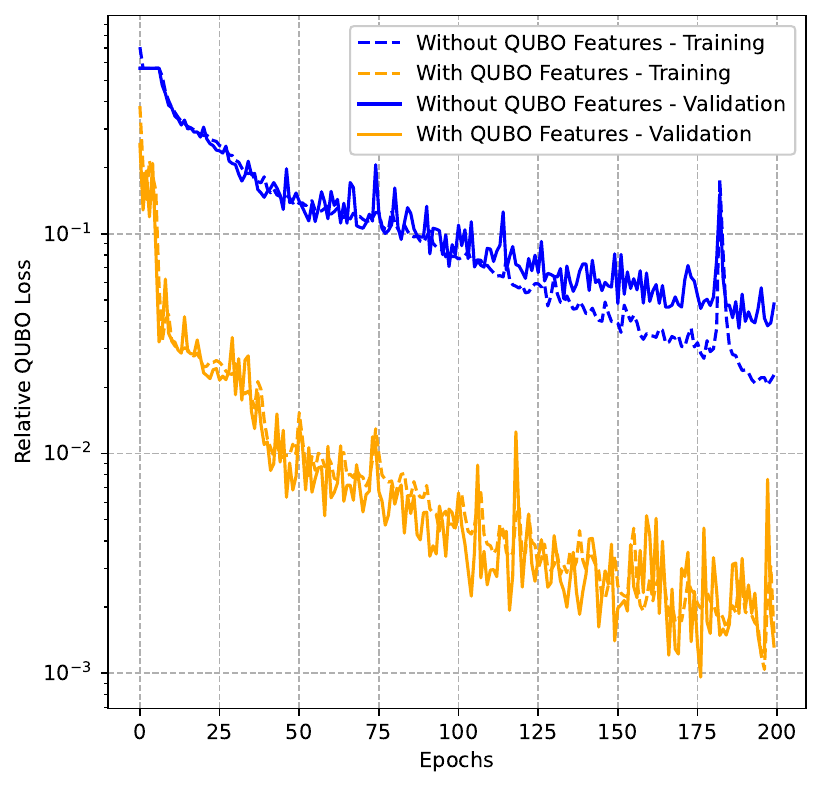}
    \caption{The \emph{relative QUBO loss} convergence in training and validation of QUBO-GNN on the small problem ($k=10)$, with and without the QUBO-aware node features from \cref{eq:quboAware}.}
    \label{fig:qubo_convergence}
    \vspace{3em}
\end{figure}

\paragraph{QUBO Aware Features Improve Performance.}
 Our QUBO-GNN incorporates QUBO-aware node features in \cref{eq:quboAware}. In Table \ref{tab:ablation_quboLoss} and in the Appendix, we examine the impact of QUBO aware features on the downstream performance. Including QUBO-aware features adds negligible computation time (less than a millisecond) while improving accuracy by an average of 1.7\%, reducing the relative QUBO objective by a factor of ~2. 
Furthermore, as we show in \cref{fig:qubo_convergence}, a beneficial property of the QUBO-aware features is their ability to accelerate the training convergence of the network.

\section{Conclusions}
\label{sec:summary}

This paper introduces a novel approach that merges graph learning with optimization problem-solving, offering a new perspective on approximating combinatorial problem solutions using GNNs. We analyze QUBO problems, revealing their piecewise constant sensitivity and heterophilic solutions, which guide the design of QUBO-GNN to approximate input-solution mappings.

We propose a scalable mechanism to efficiently generate input-solution data pairs without solving the NP-complete QUBO problem from scratch. Numerical results demonstrate QUBO-GNN's superior accuracy and computational efficiency compared to other GNN solutions and classical solvers.

    \section{Limitations and Open Challenges}
\label{sec:challenges}

Our experiments demonstrate the ability to train a network to solve the QUBO problem \cref{eq:bp} for a {\bf fixed} matrix $\bfA$ and a variable vector $\bfb$. Our approach, QUBO-GNN, generalizes to new and unseen $\bfb$ given fixed $\bfA$, a major step beyond \citet{schuetz2022combinatorial}, where a new network must be trained for every pair $(\bfA, \bfb)$. This setting, with fixed $\bfA$ and varying $\bfb$, arises frequently in practice, e.g., asset allocation in Finance \citep{herman2022survey} and computational biology \citep{marchetti2022quantum}. At the same time, while our theoretical and empirical findings establish this connection between GNNs and solving complex combinatorial problems, several aspects remain open. In particular, we identify the following research directions: 
\textbf{(i) Variable Adjacency Matrix $\bfA$.}Allowing $\bfA$ to vary will enable the modeling of different underlying graphs and problems, rather than only changes in $\bfb$. This is a well-studied challenge in non-neural approaches \cite{oshiyama2022benchmark} and remains open for neural methods; 
\textbf{(ii) Constrained Problems.} We focus on QUBO, an unconstrained subclass, but many practical problems involve constraints, e.g., knapsack \citep{kellerer2004multidimensional}, topology optimization \citep{sivapuram2018topology}, and protein design \citep{Strokach2020}. Extending GNN-based methods to handle such constraints is a key direction; and \textbf{(iii) Nonquadratic Problems.} Beyond QUBO, future work can target nonquadratic binary optimization. For instance, \cite{murray2010algorithm} proposed methods for protein design \citep{allouche2014computational} and telecommunication optimization \citep{resende2008handbook}. Extending neural architectures to this broader class would further expand our framework’s applicability.

\bibliography{biblio}

\clearpage

\appendix
\onecolumn

\section{Datasets}
\label{app:datasets}
We create 3 datasets: small ($k=10$), medium ($k=256$), and large ($k=1024$). For the small dataset, we generate exact solutions using an exhaustive search, with 1 million examples. For the small and medium examples, we generate 1000 examples using the procedure described in \cref{alg:getdata}. Since we consider problems where $\bfA$ is fixed but the observed-vector $\bfb$ is variable, we repeat the process of generating datasets 10 times, so that we have examples on several randomly generated matrices $\bfA$. The results in Section \ref{sec:experiments} are averaged across the 10 datasets, following the common experimental settings in \citet{xu2018how}. In all cases, we employ a random training/validation split of 80\%/20\%, respectively.

For the medium and large datasets, we chose $\bfA$ to be the matrix that discretizes the 2D Laplacian on a lattice.
Such matrices often appear in thermodynamics calculations \cite{EvansPDE}. The data is generated using algorithm~\ref{alg:getdata}. In order to make the problem non-trivial, we choose $\sigma=4$, which makes roughly $5\%$ of the entries in $\bfx_o$ change. We have used $1000$ Tabu search steps in step (4) of the algorithm to approximate $\bfx$.

\section{Hyperparameters}
\label{app:hyperparams}
In our experiments we consider the following hyperparameters for all GNN methods, including ours: ${\rm{learning \textunderscore rate}}\in \{1e-5,1e-4,1e-3\}$, ${\rm{weight} \textunderscore decay}\in \{1e-5,1e-4,0\}$, ${\rm{dropout}} \in \{0, 0.1, 0.5\}$. We train each network for up to 200 epochs. Dropout layers, if used, are applied after each hidden layer.
We used the Adam \cite{kingma2014adam} optimizer in all experiments. 

For the classical solvers, we used up to 1000 iterations, with early stopping if the objective function does not improve after 50 iterations, with both TS and SAB methods. We start the algorithms from an initial guess of zeros, as is standard. 

\section{Additional Results and Discussions}
\label{app:additional_results}

\subsection{Trivial Cases of QUBO Solutions.}
\label{app:extremes}
The minimization of $ \bfx^\top \bfA \bfx$ can result in a solution that is either (i) all zeros, if the entries of A are all positive, (ii) all ones, if the entries are all negative, or (iii) a mix of ones and zeros which tend to be heterophilic. This type of behavior is known as phase transitions. The extreme cases (where all the values in $A$ are uniquely signed) are easy and are usually solved via rule-based mechanisms because they yield the trivial solution (see \citet{beasley1998heuristic}). However, the heterophily designated GNNs in our experiments are known to handle both homophilic and heterophilic datasets. More critically, the difficult cases, as in the examples in our experiments, are in the phase transition zones, where indeed, the solution is heterophilic, as known from experiments in the literature and as analyzed in our paper.
\subsection{Standard Deviation of the Relative QUBO Loss}
\label{app:standard_deviations}
Due to space constraints, we provide the obtained relative QUBO loss standard deviations of the various methods in Table \ref{tab:standard_deviations}.

\begin{table*}[t]
  \caption{The Relative QUBO loss (lower is better) average and standard deviations over 10 experiments, using various naive exhaustive search methods, classical methods, GNNs, and our QUBO-GNN.}

  \setlength{\tabcolsep}{4pt} 
  \centering
  \begin{tabular}{lccccccccccc}
    \toprule
     Size $\rightarrow$ & {Small ($k=10$)} & {Medium ($k=256$)} &{Large $(k=1024)$} \\
     Method $\downarrow$ \\
    \midrule
    \textbf{\textsc{Classical solvers.}} \\
    TS & $5.90\cdot10^{-3}\pm1.45\cdot10^{-6}$ &   $4.72\cdot10^{-3}\pm5.38\cdot10^{-6}$  &  $4.58\cdot10^{-3} \pm$2.25$\cdot10^{-5}$  \\
    SAB &  0 &   $5.38\cdot10^{-2}\pm3.91\cdot10^{-4}$ &   $7.96\cdot10^{-2}\pm7.23\cdot10^{-4}$ \\
    \midrule
    \textbf{\textsc{GNN Baselines.}} \\
    GCN &   $7.12\cdot10^{-1}\pm1.29\cdot10^{-2}$ &   $5.03\cdot10^{-1}\pm3.60\cdot10^{-2}$  &  $8.76\cdot10^{-1}\pm7.21\cdot10^{-2}$  \\
    GAT &   $6.98\cdot10^{-1}\pm3.96\cdot10^{-2}$  & $4.38\cdot10^{-1}\pm5.78\cdot10^{-2}$ & $7.23\cdot10^{-1}\pm8.35\cdot10^{-2}$  \\
        GIN &  $6.02\cdot10^{-1}\pm5.10 \cdot 10^{-2}$&  $4.19\cdot 10^{-1}\pm 5.32 \cdot10^{-2}$ & $6.20 \cdot 10^{-1}\pm6.73\cdot10^{-2}$  \\
    GCNII   & $5.67\cdot10^{-1}\pm1.10\cdot10^{-2}$ &   $4.13\cdot10^{-1}\pm2.17\cdot10^{-2}$  &  $5.08\cdot10^{-1}\pm4.93\cdot10^{-2}$ \\
    Dir-GNN &  $3.74\cdot10^{-2}\pm5.85\cdot10^{-3}$  &  $1.07\cdot10^{-1}\pm9.91\cdot10^{-3}$  &  $9.62\cdot10^{-2}\pm9.97\cdot10^{-3}$ \\
    \midrule
    \textbf{\textsc{Ours.}} \\
    QUBO-GNN &  $2.83\cdot10^{-3}\pm4.48\cdot10^{-4}$ &  $5.31\cdot10^{-3}\pm5.92\cdot10^{-4}$  &  $1.26\cdot10^{-2}\pm5.47\cdot10^{-3}$  \\
    QUBO-GNN + TS  & $1.37\cdot10^{-5}\pm2.77\cdot10^{-6}$  &  $1.21\cdot10^{-5}\pm3.81\cdot10^{-6}$ &  $1.10\cdot10^{-4}\pm7.09\cdot10^{-6}$ \\
    \bottomrule
  \end{tabular}
  \label{tab:standard_deviations}
\end{table*}

\subsection{The Importance of QUBO-Aware Features}
\label{app:ablation}

We now provide additional results on the importance of including the QUBO-aware features, on small and medium-sized problems, in Tables \ref{tab:ablation_quboLoss_small} and \ref{tab:ablation_quboLoss_medium}, respectively. As discussed in \cref{sec:experiments}, the results indicate that including the QUBO-aware features improves accuracy and the relative QUBO objective.

\begin{table}[t]
\footnotesize
    \centering
        \caption{The impact of adding the QUBO aware features from \cref{eq:quboAware} to the node features on small sized problems ($k=10$).}

    \begin{tabular}{ccccc}
    \toprule
        \multirow{3}*{{QUBO Features}}&  \multicolumn{3}{c}{{Small $(k=10)$}}\\ 
       \cmidrule(r){2-4} 
      &  \multirow{2}*{Acc(\%)} & Rel. QUBO & Inference \\ & & Objective &  Time (ms)  \\     \midrule
        \xmark &  97.0$_{\pm0.6}$ & 2.18$\cdot 10^{-2}$ &  2.7 \\ 
        \cmark & 98.2$_{\pm0.5}$ & 2.83$\cdot10^{-3}$ & 2.8\\
         \bottomrule
    \end{tabular}
    \label{tab:ablation_quboLoss_small}
\end{table}

\begin{table}[t]
\footnotesize
    \centering
        \caption{The impact of adding the QUBO aware features from \cref{eq:quboAware} to the node features on medium sized problems ($k=256$).}

    \begin{tabular}{ccccc}
    \toprule
        \multirow{3}*{{QUBO Features}}&  \multicolumn{3}{c}{{Medium $(k=256)$}}\\ 
       \cmidrule(r){2-4} 
      &  \multirow{2}*{Acc(\%)} & Rel. QUBO & Inference \\ & & Objective &  Time (ms)  \\     \midrule
        \xmark &  95.9$_{\pm0.4}$ & 6.14$\cdot 10^{-2}$ &  2.9 \\ 
        \cmark & 97.2$_{\pm0.5}$ & 5.31$\cdot 10^{-3}$ & 3.0\\
         \bottomrule
    \end{tabular}
    \label{tab:ablation_quboLoss_medium}
\end{table}

\end{document}